\documentclass[twoside,11pt]{article}

\usepackage{jmlr2e}

\usepackage{booktabs}
\usepackage{float}
\usepackage[load-configurations=version-1]{siunitx}

\usepackage{amsmath}
\usepackage{graphicx}
\usepackage{color}
\usepackage{amssymb}
\usepackage{bm}
\usepackage[linesnumbered,ruled,vlined]{algorithm2e}

\DeclareMathOperator*{\argmax}{arg\,max}

\ShortHeadings{Block Sparse Canonical Correlation Analysis}{Solari, Brown and Bickel}

\begin{document}

\title{\texttt{BLOCCS}: Block Sparse Canonical Correlation Analysis With Application To Interpretable Omics Integration}

\author{\name Omid Shams Solari\thanks{corresponding author} \email solari@berkeley.edu \\
       \addr Department of Statistics\\
       University of California, Berkeley
       \AND
       \name Rojin Safavi \email rsafavi@ucsc.edu\\
       \addr Department of Bio-Engineering\\
       University of California, Santa-Cruz
       \AND
       \name James B. Brown \email jbbrown@lbl.gov\\
       \addr Lawrence Berkeley National Laboratory and Department of Statistics\\
       University of California, Berkeley}


\maketitle

\begin{abstract}
We introduce the first \textit{Sparse Canonical Correlation Analysis (sCCA)} approach which is able to estimate the leading $d$ pairs of canonical directions of a pair of datasets (together a ``block") at once, rather than the common deflation scheme, resulting in significantly improved orthogonality of the sparse directions – which translates to more interpretable solutions. We term our approach \textit{block sCCA}. Our approach builds on the sparse CCA method of \cite{solari19} in that we also express the bi-convex objective of our block formulation as a concave minimization problem, whose search domain is shrunk significantly to its boundaries, which is then optimized via gradient descent algorithm. Our simulations show that our method significantly outperforms existing sCCA algorithms and implementations in terms of computational cost and stability, mainly due to the drastic shrinkage of our search space, and the correlation within and orthogonality between pairs of estimated canonical covariates. Finally, we apply our method, available as an \texttt{R}-package called \texttt{BLOCCS}, to multi-omic data on \textit{Lung Squamous Cell Carcinoma(LUSC)} obtained via \textit{The Cancer Genome Atlas}, and demonstrate its capability in capturing meaningful biological associations relevant to the hypothesis under study rather than spurious dominant variations. 

\end{abstract}

\begin{keywords}
Multi-View Learning, Sparse Canonical Correlation Analysis, Representation Learning, Kernel Learning
\end{keywords}

\section{Introduction}

Multi-view\footnote{Each dataset, denoted by $\bm{X}_i \in \mathbb{R}^{n \times p_i}$ in this paper, containing observations on random vectors is termed a \textit{view} in this article.} observations, i.e. observations of multiple random vectors or feature sets on matching subjects-- i.e., heterogeneous datasets, are increasingly ubiquitous in data science. Particularly, in molecular biology, multiple "omics" layers are regularly collected -- measurements that sample comprehensively from an underlying pool of molecules, such as a genome, or the set of all RNA transcripts, known as the transcriptome. For example, The Cancer Genome Atlas (TCGA) is a multi-omics molecular characterization of tumors across thousands of patients. In such studies, we are often interested in understanding how two or more omics layers, or views, are related to one another -- e.g., how genotype relates to gene expression, revealing transcriptional regulatory relationships -- for a review see \cite{li2016review}. This is very different from classical regression settings, where we have a one-dimensional response that we aim to model as a function of a vector of explanatory variables. As a result, new models are needed to enable the discovery of interpretable hypotheses regarding the association structures in multi-view settings, including multi-omics.

\textit{Canonical Correlation Analysis}(CCA), \cite{hotelling}, is one set of such models whose objective is to find linear combinations of two sets of random variables such that they are maximally correlated. CCA is the most popular approach up to date in such settings which has been applied in almost all areas of science including: medicine \cite{monmonier73}, policy \cite{hopkins69}, physics \cite{wong80}, chemistry \cite{tu89}, and finance \cite{assetLiability}. Several variants of CCA to incorporate non-linear combinations of covariates, e.g. Kernel CCA of \cite{lai00} and Deep CCA of \cite{andrew13}, have also been widely particularly popular in neuro-imaging \cite{blaschko2011semi}, computer vision \cite{huang2010super}, and genetics \cite{chaudhary2018deep}.

Despite various improvements in multi-view models, inference, interpretability and model selection is still a challenge, which is mainly owed to very high-dimensional multi-view observations that become increasingly common as high-throughput measurement systems advance. Variable selection via sparsity inducing norms is a popular approach to identifying interpretable association structures in such high-dimensional settings, which are particularly important since, from a biological perspective, it is likely that responses of interest arise from the action of genes functioning in pathways. In other words, for a particular outcome, such as disease-free survival in particular cancer, not all genes are relevant, or, to use the multi-view learning parlance, "active". Hence, the derivation of sparse models from the analysis of multi-omics data is of intrinsic interest to biological data scientists. 

While several sparse CCA methods are available, \cite{witten:tibshirani:2009}, \cite{parkhomenkoSCCA}, \cite{waaijenborg}, \cite{chu:2013}, their lack of stability and empirical consistency, and additionally their high computational cost, makes them unsuitable non-parametric hypothesis testing or hyperparameter tuning. \cite{solari19} introduce \texttt{MuLe} which is a set of approaches to solving sparse CCA problems using power iterations. They demonstrate superior stability and empirical consistency compared to other popular algorithms as well as significantly lower computational cost. One shortcoming however, which is common among all sparse CCA and sparse PCA approaches, is that none guarantee, or even heuristically enforce, orthogonality between estimated canonical directions. Here, our approach also relies on power iterations; however, we address the lack of orthogonality by estimating multiple canonical directions at once -- adapting a block formulation for novel use in sparse CCA \cite{journe:nesterov}.

\section{Notation}
We term the observed random vector $X_i(\omega): \Omega \rightarrow \mathbb{R}^{p_i}$, denoted by $\bm{X}_i \in \mathbb{R}^{n \times p_i}$, $i = 1, \ldots, m$, a \textit{view}. We denote scalar, vector, and matrix parameters by lower-case normal, lower-case bold, and upper-case bold letters, respectively, and random variables by upper-case normal letters. $n$ is used to indicate the sample size and $p_i$ the dimensionality of the covariate space of each of $m$ views. Canonical directions are denoted by $\bm{z}_i \in \mathcal{B}^{p_i}$, or $\bm{z}_i \in \mathcal{S}^{p_i}$, and $\bm{Z}_i \in \mathcal{S}_d^{p_i}$, where $\mathcal{B} = \{\bm{x} \in \mathbb{R} | \| \bm{x}\|_2 \leq 1 \}$ and $\mathcal{S} = \{\bm{x} \in \mathbb{R} | \| \bm{x}\|_2 = 1 \}$. $\mathcal{S}_d^{p} = \{ \bm{Z} \in \mathbb{R}^{p \times d} | \bm{Z}^{\top}\bm{Z} = \bm{I}_d \}$ denotes a \textit{Stiefel} manifold which is the set of all d-frames, i.e. the space of ordered sets of $d$ linearly independent vectors, in $\mathbb{R}^p$. $l_x(\bm{z}): \mathbb{R}^{p} \rightarrow \mathbb{R}$ denotes any norm function, more specifically $l_{0/1}(\bm{z}) = \| \bm{z} \|_{0/1}$, and $\bm{\tau}^{(i)}$ refers to the $i-th$ non-zero element of the vector which is specifically used for the sparsity pattern vector. We also introduce \textit{accessory variables} in Section \ref{subsec:directed} to term variables towards which we direct estimated canonical directions, neglecting their inferential role as covariates or dependent variables. We also use ``program" to refer to ``optimization programs".

\section{Background}
\label{sec:pDefinition}
Sub-space learning is perhaps the most popular concept in multi-view learning, and implies a \textit{Latent Space} generative model, where each view, $X_i(\omega): \mathcal{U} \rightarrow \mathcal{X}_i, i = 1, \ldots, m$, is assumed to be a function of a common unobservable random vector, $U: \Omega \rightarrow \mathcal{U} $ in the latent space. The main objective in subspace learning is to estimate the inverse of these mappings within a functional family, $\mathcal{F}_i = \{ F_i:\mathcal{X}_i \rightarrow \mathcal{U} \}$ assuming invertibility. At the sample level, this is interpreted as estimating $F_i(X_i)$ by $\bm{F}_i:\mathbb{R}^{n \times p_i} \rightarrow \mathcal{U}^{n}$ such that $\mathcal{S}:\mathcal{U}^{n \times m} \rightarrow \mathbb{R}^{d}$, $\mathcal{S} = (s_1, \ldots, s_d)$, where $s(F_1(X_1), \ldots, F_m(X_m)):\mathcal{U}^{n \times m} \rightarrow \mathbb{R}$ is some similarity measure between these transformed observed views is maximized,

\begin{equation}
    \bm{F}^*  = \argmax_{\substack{F_i\in \mathcal{F}_i\\ i \in \{1, \ldots, m\}}} \mathcal{S}(F_1(X_1), \ldots, F_m(X_m))
\end{equation}

Where $\bm{F} = (F_1, \ldots, F_m)$. $d$ is the number of dimensions in which similarity is maximized, which is of importance since here we are concerned with block algorithms where $d > 1$, i.e. we estimate $d$ distinct  mappings for each view at the same time such that these mappings maximize $\mathcal{S}$. In the rest of this section and most of Section \ref{sec:block} we assume that we observe only a pair of views, i.e. $m = 2$. Throughout this paper we also assume that $U: \Omega \rightarrow \mathbb{R}^k$, $X_i: \mathbb{R}^k \rightarrow \mathbb{R}^{p_i}$.

\subsection{Canonical Correlation Analysis}
If we assert the functional families $\mathcal{F}_i$ to be a subset of the parametric family of linear functions $\mathcal{L} = \{l_i:\mathbb{R}^{p_i} \rightarrow \mathbb{R}^k, l_i(X_i) = \bm{z}_i X_i \} $, and the similarity criterion to be the Pearson correlation, we end up with the \textit{Canonical Correlation Analysis} criterion. Assuming $E[X_1] = \bm{0}^{p_1}$ and $E[X_2] = \bm{0}^{p_2}$,

\begin{equation}
\label{eq:cca}
\begin{split}
    (\bm{z}_1^*, \bm{z}_2^*)  &= \argmax_{\bm{z}_1\in \mathbb{R}^{p_1}, \bm{z}_2\in \mathbb{R}^{p_2}} \rho(X_1\bm{z}_1, X_2\bm{z}_2)\\  &= \argmax_{\bm{z}_1\in \mathbb{R}^{p_1}, \bm{z}_2\in \mathbb{R}^{p_2}} \frac{E[(X_1\bm{z}_1)^{\top}(X_2\bm{z_2})]}{E[(X_1\bm{z}_1)^2]^{1/2}E[(X_2\bm{z_2})^2]^{1/2}}
\end{split}
\end{equation}

Since we almost always have access only to samples from $X_1$ and $X_2$, we estimate Program \ref{eq:cca} using plug-in sample estimators for population terms. 

\begin{equation}
\label{eq:samplecca}
    (\bm{z}_1^*, \bm{z}_2^*) = \argmax_{\bm{z}_1\in \mathbb{R}^{p_1}, \bm{z}_2\in \mathbb{R}^{p_2}} \frac{\bm{z}_1^{\top}\bm{X}_1^{\top} \bm{X}_2 \bm{z}_2}{\sqrt{\bm{z}_1^{\top}\bm{X}_1^{\top} \bm{X}_1 \bm{z}_1}\sqrt{\bm{z}_2^{\top}\bm{X}_2^{\top} \bm{X}_2 \bm{z}_2}}
\end{equation}

$\bm{z}_i$ are termed \textit{Canonical Loading Vectors} and $\bm{X}_i\bm{z}_i$ are called the \textit{Canonical Covariates}.

\section{Block Reformulations of CCA Models}
\label{sec:block}

Generalizing Program \ref{eq:samplecca} to $\bm{Z}_i \in \mathbb{R}^{p_i \times d}$, 

\begin{equation}
\label{eq:sampleblockcca}
    (\bm{Z}_1^*, \bm{Z}_2^*) = \argmax_{\substack{\bm{Z}_1\in \mathbb{R}^{p_1 \times d}, \bm{Z}_2\in \mathbb{R}^{p_2 \times d} \\ \bm{Z}_1^{\top}\bm{X}_1^{\top}\bm{X}_1\bm{Z}_1 = \bm{Z}_2^{\top}\bm{X}_2^{\top}\bm{X}_2\bm{Z}_2 = \bm{I}^{d} } } tr(\bm{Z}_1^{\top}\bm{X}_1^{\top} \bm{X}_2 \bm{Z}_2)
\end{equation}

Here we reserve the term \textit{"block formulation"} to discuss settings in which $d > 1$, i.e. we estimate multiple pairs of canonical directions at once, $\bm{Z}_i \in \mathbb{R}^{p_i \times d}$ $i = 1,2$ rather than a single pair of canonical directions $\bm{Z}_i \in \mathbb{R}^{p_i}$, $i = 1,2$. As is customary in the sparse CCA literature, here we also assume that the covariance matrix of each random vector is diagonal, i.e. $\bm{X}_i^{\top}\bm{X}_i = \bm{I}^{p_i}, i = 1,2$, which is justified in  \cite{dudoit2002comparison}. This enables us to rewrite Program \ref{eq:sampleblockcca} as,

\begin{equation}
\label{eq:sampleblockcca1}
    (\bm{Z}_1^*, \bm{Z}_2^*) = \argmax_{\substack{\bm{Z}_1 \in \mathcal{S}_d^{p_1}\\ \bm{Z}_2 \in \mathcal{S}_d^{p_2} } } tr(\bm{Z}_1^{\top}\bm{X}_1^{\top} \bm{X}_2 \bm{Z}_2)
\end{equation}

Where $\mathcal{S}_d^{p_1}$ is a \textit{Stiefel Manifold}\footnote{$\mathcal{S}_m^p = \{ \bm{M} \in \mathbb{R}^{p \times d} | \bm{M}^{\top}\bm{M} = \bm{I}  \}$}

\subsection{Regularized Block CCA}

We analyze the following generalized formulation of the sparse block CCA problem in this section,

\begin{equation}
    \label{eq:lblockcca}
    \begin{split}
        \phi_{l, d}(\bm{\gamma}_1, \bm{\gamma}_2) :=& \max_{\substack{\bm{Z}_1 \in \mathcal{S}_d^{p_1}\\ \bm{Z}_2 \in \mathcal{S}_d^{p_2} } } tr(\bm{Z}_1^{\top}\bm{C}_{12}\bm{Z}_2\bm{N})\\
        &- \sum_{j = 1}^d \gamma_{1j} l( \bm{z}_{1j}) - \sum_{j = 1}^d \gamma_{2j} l (\bm{z}_{2j} )
    \end{split}
\end{equation}

$\bm{\gamma}_i \in \mathbb{R}^d, \bm{\gamma}_i \geq 0$ is the sparsity parameter vector for each view, and $\bm{N} = diag(\bm{\mu}), \bm{\mu} \in \mathbb{R}^+$, where $d$ is the number of canonical covariates. $l(\bm{z}_{ij})$ is some norm of the $j-th$ column of the $i-th$ view, and $\bm{C}_{12}$ is the sample covariance matrix.

\begin{remark}
In practice, distinct $\mu_i$ enforces the objective in Program \ref{eq:lblockcca} to have distinct maximizers \cite{journe:nesterov}.
\end{remark} 

\subsubsection{\texorpdfstring{$L_1$}{TEXT} Regularization}
\label{subsub:l1}

Here we consider Program \ref{eq:lblockcca} with $L_1$ regularization, and decouple the problem along multiple canonical directions resulting in the following program,

\begin{equation}
    \label{eq:decouple}
    \begin{split}
        \phi_{l_1, d}(\bm{\gamma}_1, \bm{\gamma}_2) =& \max_{\bm{Z}_1 \in \mathcal{S}_d^{p_1}} \sum_{j = 1}^{d} \max_{\bm{z}_{2j} \in \mathcal{S}^{p_2} } [ \mu_j \bm{z}_{1j}^{\top}\bm{C}_{12}\bm{z}_{2j} - \gamma_{2j} \| \bm{z}_{2j} \|_1 ]\\
        &- \sum_{j = 1}^d \gamma_{1j} \| \bm{z}_{1j}\|_1
    \end{split}
\end{equation}

where $\bm{z}_{ij}$ is the $j$-th column of the $i$-th dataset.

\begin{theorem}
\label{thm:l1}
Maximizers $\bm{Z}_1^*$ and $\bm{Z}_2^*$ of Program \ref{eq:decouple} are,

\begin{equation}
    \label{eq:l1z1}
    \bm{Z}_1^* = \argmax_{\bm{Z}_1 \in \mathcal{S}^{p_1}_{d}} \sum_{j = 1}^d \sum_{i = 1}^{p_2} [ \mu_j |\bm{c}_i^{\top}\bm{z}_{1j}| - \gamma_{2j} ]_+^{2} - \sum_{j = 1}^d \gamma_{1j} \| \bm{z}_{1j}\|_1
\end{equation}

and,

\begin{equation}
    \label{eq:l1z2}
    [\bm{Z}_2]_{ij}^* = \frac{sgn(\bm{c}_i^{\top}\bm{z}_{1j})[ \mu_j|\bm{c}_i^{\top}\bm{z}_{1j}| - \gamma_{2j} ]_+}{\sqrt{\sum_{k = 1}^{p_2} [ \mu_j|\bm{c}_k^{\top}\bm{z}_{1j}| - \gamma_{2j} ]_+^2}}
\end{equation}

\end{theorem}

Equation \ref{eq:l1z2} is utilized to derive the necessary and sufficient conditions under which $z_{2ji}^*$ is active, i.e. inferring the sparsity pattern matrix, $supp(\bm{Z})$, which is denoted her by $\bm{T}_{2} \in \{0,1\}^{p_2 \times d}$.

\begin{corollary}
\label{cor:l1}
$[\bm{T}_{2}]_{ij} = 0$, i.e. $z_{2ji}^* \in supp(\bm{Z}_2^*)$, iff $| \bm{c}_i^{\top}\bm{z}_{1j}^* | \leq \gamma_{2j}/\mu_j$.
\end{corollary}

Theorem \ref{thm:l1} enables us to infer the the sparsity pattern of either of the canonical directions due to the symmetry of the problem. Assuming we estimate $\bm{T}_2$ first, we shrink the sample covariance matrix to $[\bm{C}_{12}']_{kl} = [\bm{C}_{12}]_{k\bm{\tau}_{2j}^{(l)}}$ where $\bm{\tau}_{2j}^{(l)}$ is the $l$-th non-zero element of the $j$-th column of $\bm{T}_2$. We then use this reduced covariance matrix to estimate $\bm{T}_1$. Having estimated the sparsity pattern matrices in the first stage, we estimate the active elements of the canonical direction matrices in the second stage by first shrinking the covariance matrix on both sides, resulting in $[\bm{C}_{12}^{(j)}]_{kl} = [\bm{C}_{12}]_{\bm{\tau}_{1j}^{(k)},\bm{\tau}_{2j}^{(l)}}$, then estimating its active elements via an alternating algorithm introduced in \ref{subsec:postprocess}.

\begin{remark}
\label{rmk:approx}
According to Theorem \ref{thm:l1}, in order to infer the sparsity pattern matrices, we need to maximize Program \ref{eq:l1z2}. This program is non-convex; however we approximate it by ignoring the penalty term which turns it into the following concave minimization over the unit sphere, 

\begin{equation}
    \label{eq:l1approx}
    \phi_{l_1, d}(\bm{\gamma}_1, \bm{\gamma}_2) = \max_{\bm{Z}_1 \in \mathcal{S}^{p_1}_{d}} \sum_{j = 1}^d \{ \sum_{i = 1}^{p_2} [ \mu_j |\bm{c}_i^{\top}\bm{z}_{1j}| - \gamma_{2j} ]_+^{2} \}
\end{equation}

which is solved using a simple gradient ascent algorithm. It is important to note that this approximation is justifiable. Our simulations demonstrate that this approximation does not affect the capability of our approach to uncover the support of our underlying generative model. Secondly, as we have mentioned in Corollary \ref{cor:l1}, we use the optima of this program in the first stage to infer the sparsity patterns of canonical directions. Also we can show that for every $(\gamma_{1j}, \gamma_{2j})$ that results in $\bm{z}_{1j}^* = 0$ according to the Corollary \ref{cor:l1}, there is a $\gamma_{2j}' \geq \gamma_{2j}$ in Program \ref{eq:l1approx} for which $z_{2ji}^* = 0$.
\end{remark}

In the rest of this section we introduce \textit{Block Sparse Multi-View CCA} and \textit{Block Sparse Directed CCA}.

\subsection{\texorpdfstring{$L_1$}{TEXT} Regularized Block Multi-View CCA}
\label{subsec:multimodal}

Now we extend our approach from \ref{subsub:l1} to identify correlation structures between more than two views, $\bm{X}_i \in \mathbb{R}^{n \times p_i}, i = 1, \ldots, m$. The application of such methods are ever-increasing, e.g. understanding the enriched genetic pathways in a population of patients with a specific type of cancer. We extend the approach introduced in \cite{solari19} to our block setting, which results in the following optimization program,

\begin{equation}
\label{eq:mCCA}
\begin{split}
    \phi_{l_1,d}^m(\bm{\Gamma}_1, \ldots, \bm{\Gamma}_d) =& \max_{\substack{\bm{Z}_i \in \mathcal{S}_d^{p_i}\\ \forall i = 1, \ldots, m }} \sum_{r<s = 2}^{m} tr(\bm{Z}_r^{\top}\bm{C}_{rs}\bm{Z}_s\bm{N})\\
    &- \sum_{j = 1}^d \sum_{s = 2}^m  \sum_{\substack{r = 1 \\ r \neq s }}^{s-1} \gamma_{srj} \| \bm{z}_{sj} \|_1 
\end{split}
\end{equation}

where $\bm{\Gamma}_j \in [0,1]^{p_j \times M}$ are the sparsity parameter matrices whose elements $\gamma_{srj}$ regulate the sparsity of canonical direction $\bm{z}_{sj}$ in relation to $\bm{z}_{rj}$, where $\bm{z}_{sj}$ is the $j$-th column of $\bm{Z}_s$. As before $\bm{C}_{rs} = 1/n \bm{X}_r^{\top}\bm{X}_s$ is a sample covariance matrix.

\begin{theorem}
\label{thm:mCCA}
Maximizers $\bm{Z}_i^*, i = 1, \ldots, m$ of Program \ref{eq:mCCA} are,

\begin{equation}
\label{eq:mZsstar}
\begin{split}
    z_{sij}^* & (\gamma_{sr1},\ldots ,\gamma_{srd}) =\\ &\frac{sgn(\sum_{ \substack{r = 1\\ r \neq s}}^m  \tilde{\bm{c}}_{rsi}^{\top} \bm{z}_{rj}) [\mu_j|\sum_{ \substack{r = 1\\ r \neq s}}^m  \tilde{\bm{c}}_{rsi}^{\top} \bm{z}_{rj}| - \sum_{\substack{r = 1\\ r \neq s}}^{m} \gamma_{srj}]_+}{\sqrt{\sum_{k=1}^{p_2} [\mu_j|\sum_{ \substack{r = 1\\ r \neq s}}^m  \tilde{\bm{c}}_{rsk}^{\top} \bm{z}_{rj}| - \sum_{\substack{r = 1\\ r \neq s}}^{m} \gamma_{srj}]_+^2 } }
\end{split}
\end{equation}

and for $r = 1, \ldots, m$ and $r \neq s$,

\begin{equation}
\label{eq:mZR}
\begin{split}
    \bm{Z}_r^* & (\bm{\Gamma}_1, \ldots, \bm{\Gamma}_d) =\\ & \argmax_{\substack{\bm{Z}_r \in \mathcal{S}^{p_r}_d\\ r \neq s, r = 1, \ldots, m} }  \sum_{j = 1}^d \sum_{i= 1}^{p_s} [\mu_j|\sum_{ \substack{r = 1\\ r \neq s}}^m  \tilde{\bm{c}}_{rsi}^{\top} \bm{z}_{rj}| - \sum_{\substack{r = 1\\ r \neq s}}^{m} \gamma_{srj}]_+^2 +\\  
    &\sum_{ \substack{i < r = 2\\ i, r \neq s }}^m tr(\bm{Z}_i^{\top}\bm{C}_{ir}\bm{Z}_r\bm{N})  -\sum_{j = 1}^d\sum_{\substack{i = 1 \\ i \neq s} }^m  \sum_{\substack{r = 1 \\ i \neq j }}^{s-1} \gamma_{irj} \| \bm{z}_{ij} \|_1
\end{split}
\end{equation}

\end{theorem}

Similar to the previous section, we drop the last term in Program \ref{eq:mZR} following the same justifications offered in Remark \ref{rmk:approx}. This approximation leaves us with a concave minimization program which can be solved in a significantly faster and more stable way.

\begin{corollary}
\label{cor:mccasparsityl0}
Given the sparsity parameter matrices $\bm{\Gamma}_i, i = 1, \ldots, d$ and the solution, $\bm{Z}_r^*$ for $r = 1, \ldots, m$ and $r \neq s$, to the Program \ref{eq:mZR},

\begin{equation}
 \label{eq:mccaspl1}
[\bm{T}_{s}]_{ij} = \begin{cases}
0 & |\sum_{ \substack{r = 1\\ r \neq s}}^m  \tilde{\bm{c}}_{rsi}^{\top} \bm{z}_{rj}| \leq 1/\mu_j \sum_{\substack{r = 1\\ r \neq s}}^{m} \gamma_{srj}\\
1 & otherwise
\end{cases}
\end{equation}

\end{corollary}

\subsection{\texorpdfstring{$L_1$}{TEXT} Regularized Directed CCA}
\label{subsec:directed}

Often samples involved in a multi-view learning problem are part of a designed experiment which differ along the direction of some treatment vector, or an observational study where we have information about the samples in addition to the observed views, e.g. socioeconomic status, sex, education level, etc. \cite{solari19} coined the term \textit{Accessory Variable} to avoid confusions with the rich lexicon of statistical inference, to point out that this extra piece of information will be solely used to direct canonical directions such that they capture correlation structures which also align with these accessory variables, denoted here by $\bm{Y} \in \mathbb{R}^{n\times d}$, towards each column of which we direct the canonical directions. To this end, we form the following optimization problem,

\begin{equation}
\label{eq:blockdirected}
    \begin{split}
    \phi_{l, d}(\bm{\gamma}_1, \bm{\gamma}_2, \bm{\epsilon}_1, \bm{\epsilon}_2) =& \max_{\substack{\bm{Z}_1 \in \mathcal{S}_d^{p_1}\\ \bm{Z}_2 \in \mathcal{S}_d^{p_2} } } tr(\bm{Z}_1^{\top}\bm{C}_{12}\bm{Z}_2\bm{N})\\
    &- \sum_{i = 1}^2[\mathcal{L}(\bm{X}_i\bm{Z}_{i}\bm{N}\bm{E}_i, \bm{Y}) + \bm{\gamma}_i^{\top}\bm{l}(\bm{Z}_i) ]
    \end{split}
\end{equation}

where $\bm{E}_i = diag(\bm{\epsilon}_i)$ are diagonal hyper-parameter matrices controlling the effect of the accessory variables on the canonical directions. $\mathcal{L}(\bm{A}, \bm{B}):\mathcal{X}_{A} \times \mathcal{X}_{B} \rightarrow \mathbb{R}$ is a measure of column-wise misalignment of $\bm{A}$ and $\bm{B}$. Here, we choose the Euclidean inner-product as our alignment measure, i.e. $\mathcal{L}(\bm{X}_i\bm{Z}_{i}\bm{N}\bm{E}_i, \bm{y}) = - \langle \bm{X}_i\bm{Z}_{i}\bm{N}\bm{E}_i, \bm{Y} \rangle = -tr(\bm{Y}^{\top}\bm{X}_i\bm{Z}_{i}\bm{N}\bm{E}_i)$. Plugging in \ref{eq:blockdirected} and decoupling,

\begin{align}
        \phi_{l_1, d}(\bm{\gamma}_1, \bm{\gamma}_2) = \max_{\bm{Z}_1 \in \mathcal{S}_d^{p_1}} \sum_{j = 1}^{d} \max_{\bm{z}_{2j} \in \mathcal{S}^{p_2} } [ \mu_j \bm{z}_{1j}^{\top}\bm{C}_{12}\bm{z}_{2j} \nonumber\\ +(\mu_j \epsilon_{1j} \bm{y}_j^{\top}\bm{X}_2\bm{z}_{2j} - \gamma_{2j} \| \bm{z}_{2j} \|_1) ] \nonumber\\
        +\sum_{j = 1}^d (\mu_j\epsilon_{2j}\bm{y}_j^{\top}\bm{X}_1\bm{z}_{1j} - \gamma_{1j} \| \bm{z}_{1j}\|_1) \label{eq:directeddecouple}
\end{align}

where $\bm{z}_{ij}$ is the $j$-th column of the $i$-th dataset.

\begin{theorem}
\label{thm:directedl1}
Maximizers of Program \ref{eq:directeddecouple} are,

\begin{align}
    \label{eq:directedl1z1}
    \bm{Z}_1^* = \argmax_{\bm{Z}_1 \in \mathcal{S}^{p_1}_{d}} \sum_{j = 1}^d \sum_{i = 1}^{p_2} [ \mu_j |\bm{c}_i^{\top}\bm{z}_{1j} + \epsilon_{2j}\bm{x}_{2i}^{\top}\bm{y}_j | - \gamma_{2j} ]_+^{2} \\
    +\sum_{j = 1}^d (\mu_j\epsilon_{1j}\bm{y}_j^{\top}\bm{X}_1\bm{z}_{1j} - \gamma_{1j} \| \bm{z}_{1j}\|_1)
\end{align}

and,

\begin{align}
    \label{eq:directedl1z2}
    [\bm{Z}_2]_{ij}^* &=\\ &\frac{sgn(\bm{c}_i^{\top}\bm{z}_{1j}+ \epsilon_{2j}\bm{x}_{2i}^{\top}\bm{y}_j)[ \mu_j|\bm{c}_i^{\top}\bm{z}_{1j}+\epsilon_{2j}\bm{x}_{2i}^{\top}\bm{y}_j| - \gamma_{2j} ]_+}{\sqrt{\sum_{k = 1}^{p_2} [ \mu_j|\bm{c}_k^{\top}\bm{z}_{1j}+\epsilon_{2j}\bm{x}_{2k}^{\top}\bm{y}_j| - \gamma_{2j} ]_+^2}}
\end{align}

\end{theorem}

In the following corollary we formalize the necessary and sufficient conditions under which $z_{2ij}^*$ is active using Equation \ref{eq:directedl1z2}.

\begin{corollary}
\label{cor:directedl1}
$[\bm{T}_{2}]_{ij} = 0$, iff $| \bm{c}_k^{\top}\bm{z}_{1j}^*+\epsilon_{2j}\bm{x}_{2k}^{\top}\bm{y}_j | \leq \gamma_{2j}/\mu_j$.
\end{corollary}

In the following section we propose algorithms to solve the optimization programs discussed so far.

Please refer to the \textit{Supplementals} for detailed proofs of the theorems and corollaries presented above as well as a discussion of \textit{$l_0$-regularized} Canonical Correlation Analysis.


\section{\texttt{BLOCCS}: Gradient Ascent Algorithms for Regularized Block Models}

As discussed so far, we reformulated each of the four cases studied into a concave minimization program over a Stiefel manifold. Our proposed algorithms involve a simple first-order optimization method at their cores, see \textit{Supplementals}. In \ref{subsec:sparsity} we apply this first-order method to the scenarios discussed so far, which constitutes the first stage of our two-stage approach. In the first stage, we estimate the sparsity patterns of our canonical directions. In the second stage we estimate the ``active" entries (non-zero loadings) of the canonical directions using an alternating optimization algorithm discussed in \ref{subsec:postprocess}.

\subsection{Sparsity Pattern Estimation}
\label{subsec:sparsity}

In the first stage we estimate the sparsity patterns of the canonical directions, $\bm{T}_i$, by applying each of the following algorithms once for each dataset. As we move from estimating $\bm{T}_1$ to $\bm{T}_m$, we use a technique which we term \textit{Successive Shrinking}, that is having estimated $\bm{T}_i$, we shrink every sample covariance matrix $\bm{C}_{ij}, j\neq i$ to $[\bm{C}_{ij}']_{rs} = [\bm{C}_{ij}]_{\bm{\tau}_{ik}^{(r)}s}$, where $\bm{\tau}_{ik}^{(r)}$ is the $r$-th non-zero element of the $k$-th column of the $i$-th sparsity pattern matrix. As a result, in each successive shrinkage the covariance matrices are shrunk drastically, which in turn results in significant speed-up of our algorithm.

\subsubsection{\texorpdfstring{$L_1$}{TEXT} Regularized Algorithm}
\label{subsubsec:l1alg}

Now we apply our first-order maximization algorithm to Program \ref{eq:l1approx},

\begin{algorithm}
 \KwData{Sample Covariance Matrix $\bm{C}_{12}$\\  \quad \qquad    Regularization parameter vector $\bm{\gamma_2} \in [0,1]^{d}$ \\ \quad \qquad Initialization $\bm{Z}_1 \in \mathcal{S}_d^{p_1}$\\ \quad \qquad $\bm{N} = diag(\mu_1, \ldots, \mu_d) \succ 0$\\
 \quad \qquad (optional) $\bm{T}_1 \in \{0, 1\}^{p_1 \times d}$}
 \KwResult{ $\bm{T}_2$, optimal sparsity pattern of $\bm{Z}_2^*$}
 initialization\;
 
 \While{ convergence criterion is not met }{
          \For{$j = 1, \ldots, d$}{$\bm{z}_{1j} \leftarrow \sum_{i = 1}^{p_2} \mu_j[\mu_j| \bm{c}_i^{\top}\bm{z}_{1j} | - \gamma_2]_+ sgn(\bm{c}_i^{\top}\bm{z}_{1j})\bm{c}_i$}
    $\bm{Z}_1 \leftarrow polar(\bm{Z}_1)$\\
    \If{$\bm{T}_1$ is given}{$\bm{Z}_1 \leftarrow \bm{Z}_1 \circ \bm{T}_1$}
    }
 
 Output $\bm{T}_2 \in \{0,1\}^{p_2 \times d}$ where $[\bm{T}_2]_{ij} = 0 $ if $|\bm{c}_i^{\top}\bm{z}_{1j}^*| \leq \gamma_{2j}/\mu_j$ and 1 otherwise.\\
 
 \caption{ \texttt{BLOCCS} algorithm for solving Program \ref{eq:l1approx} }
 \label{alg:l12nd}
\end{algorithm}

As we pointed out above, we then compute $\bm{T}_1$ using successive shrinkage.

\begin{remark}
\label{rmk:algorithm}
One of the appealing qualities of our algorithm is that it is solely dependent on a function which can evaluate power iterations, which can be implemented very efficiently by exploiting sparse structures in the data matrix and canonical directions. This quality is significantly rewarded by successive shrinkage. It can also very easily be deployed on a distributed computing infrastructure. \cite{solaridunc19} utilize this quality to offer a Spark-based distributed regularized multi-view learning package.
\end{remark}

\subsubsection{Multi-View Block Sparse Algorithm}
\label{subsubsec:mCCAalg}

We now propose an algorithm to solve Program \ref{eq:mZR}, leaving out the regularization term in the first stage.

\begin{algorithm}
 \KwData{Sample Covariance Matrices $\bm{C}_{rs}, \quad 1 \leq r < s \leq m$\\  \quad \qquad Sparsity parameter matrices $\bm{\Gamma}_j \in [0,1]^{m \times m}$ for $j = 1 , \ldots, d$\\ \quad \qquad Initial values $\bm{Z}_r \in \mathcal{S}_d^{p_r}, \quad 1 \leq r \leq m$\\ \quad \qquad $\bm{N} = diag(\mu_1, \ldots, \mu_d) \succ 0$\\
 \quad \qquad (optional) $\bm{T}_r \in \{0, 1\}^{p_r \times d}, r \neq s$ \quad \qquad }
 \KwResult{ $\bm{T}_s$, optimal sparsity pattern for $\bm{Z}_s$}
 initialization\;
 \While{ convergence criterion is not met }{
 \For{$r = 1, \ldots, m$, $r \neq s$}{
    \For{$j = 1, \ldots, d$}
    {$\bm{z}_{rj} \leftarrow \sum_{i= 1}^{p_s} \mu_j[\mu_j | \sum_{ \substack{r = 1\\ r \neq s}}^m  \tilde{\bm{c}}_{rsi}^{\top} \bm{z}_{rj} | - \sum_{\substack{r = 1\\ r \neq s}}^{m} \gamma_{srj} ]_+ sgn( \sum_{ \substack{r = 1\\ r \neq s}}^m  \tilde{\bm{c}}_{rsi}^{\top} \bm{z}_{rj}) \tilde{\bm{c}}_{rsi} + \mu_j \sum_{\substack{l = 1\\ l \neq r,s}}^{m} \tilde{\bm{C}}_{rl}\bm{z}_{lj} $}
    
    $\bm{Z}_r \leftarrow polar(\bm{Z}_r)$\\
    \If{$\bm{T}_r$ is given}{$\bm{Z}_r \leftarrow \bm{Z}_r \circ \bm{T}_r$}
    }
 }
 
 Output $\bm{T}_s \in \{0,1\}^{p_s \times d}$, $[T_{s}]_{ij} = 0 $ if $| \sum_{ \substack{r = 1\\ r \neq s}}^m  \tilde{\bm{c}}_{rsi}^{\top} \bm{z}_{rj} | \leq  1/\mu_j \sum_{\substack{r = 1\\ r \neq s}}^{m} \gamma_{srj}$ and 1 otherwise.\\
 
 \caption{\texttt{BLOCCS} algorithm for solving Program \ref{eq:mZR}}
 \label{alg:mCCA3rd}
\end{algorithm}

\subsubsection{Directed Block Regularized Algorithm}
\label{subsubsec:directedalg}

Before we present our algorithm, it is helpful to realize that the directed regularized case in Program \ref{eq:directeddecouple} is equivalent to the multi-modal case in Program \ref{eq:mCCA} with $m = 3$ and $\bm{\epsilon}_i = \bm{1}_d$. As though we regard the accessory variable $\bm{y}$ as a third view. But many times the researcher wants to have a direct control on how much effect the accessory variable will have on the canonical directions. Basically the larger $\epsilon_{ij}$, the smaller the aperture of the convex cone that contains both $\bm{y}$ and the canonical covariate $\bm{X}_i\bm{z}_i$. Below is the algorithm we devised for this problem,

\begin{algorithm}
 \KwData{Sample Covariance Matrix $\bm{C}_{12}$\\  \quad \qquad    Regularization parameter vector $\bm{\gamma_2} \in [0,1]^{d}$ \\ \quad \qquad Hyper-parameter vectors $\bm{\epsilon}_i \in \mathbb{R}^{d}, i = 1,2$ \\\quad \qquad Initialization $\bm{Z}_1 \in \mathcal{S}_d^{p_1}$\\ \quad \qquad $\bm{N} = diag(\mu_1, \ldots, \mu_d) \succ 0$\\
 \quad \qquad (optional) $\bm{T}_1 \in \{0, 1\}^{p_1 \times d}$}
 \KwResult{ $\bm{T}_2$, optimal sparsity pattern of $\bm{Z}_2^*$}
 initialization\;
 
 \While{ convergence criterion is not met }{
          \For{$j = 1, \ldots, d$}{$\bm{z}_{1j} \leftarrow \sum_{i = 1}^{p_2} \mu_j[\mu_j| \bm{c}_i^{\top}\bm{z}_{1j} + \epsilon_{2j}\bm{x}_{2i}^{\top}\bm{y}| - \gamma_2]_+ sgn(\bm{c}_i^{\top}\bm{z}_{1j}+ \epsilon_{2j}\bm{x}_{2i}^{\top}\bm{y})\bm{c}_i  + \epsilon_{1j} \bm{X}_1^{\top}\bm{y}$ }
    $\bm{Z}_1 \leftarrow polar(\bm{Z}_1)$\\
    \If{$\bm{T}_1$ is given}{$\bm{Z}_1 \leftarrow \bm{Z}_1 \circ \bm{T}_1$}
    }
 
 Output $\bm{T}_2 \in \{0,1\}^{p_2 \times d}$ where $[\bm{T}_2]_{ij} = 0 $ if $|\bm{c}_i^{\top}\bm{z}_{1j}^* + \epsilon_{2j}\bm{x}_{2i}^{\top}\bm{y}| \leq \gamma_{2j}/\mu_j$ and 1 otherwise.\\

 \caption{ \texttt{BLOCSS} algorithm for solving Program \ref{eq:directeddecouple} }
 \label{alg:directed4th}
\end{algorithm}

In Section \ref{subsec:tcga}, we demonstrate  the capabilities of this approach in exploratory data analysis and hypothesis development. 

\subsection{Active Entry Estimation}
\label{subsec:postprocess}

In the second stage of the algorithm, we estimate the active elements of the canonical directions for which, following \cite{journe:nesterov}, we also propose alternating algorithm to solve the following optimization program,

\begin{equation}
    \label{eq:2ndStage}
    \phi_{d, 0} = \max_{\substack{\bm{Z}_1 \in \mathcal{S}_d^{p1}, \bm{Z}_1|_{\neq 0} = \bm{T}_1\\
    \bm{Z}_2 \in \mathcal{S}_d^{p2}, \bm{Z}_2|_{\neq 0} = \bm{T}_2}} tr( \bm{Z}_1^{\top}\bm{C}_{12}\bm{Z}_2\bm{N} )
\end{equation}

\begin{algorithm}
 \KwData{Sample Covariance Matrix $\bm{C}_{12}$\\ \quad \qquad Initialization $\bm{Z}_i \in \mathcal{S}_d^{p_i}$ for $i = 1,2$\\ \quad \qquad $\bm{N} = diag(\mu_1, \ldots, \mu_d) \succ 0$\\
 \quad \qquad $\bm{T}_i \in \{0, 1\}^{p_i \times d}$ for $i  = 1,2$}
 \KwResult{ $\bm{Z}_i^*, i = 1,2$, local maximizers of \ref{eq:2ndStage}}
 initialization\;
 
 \While{ convergence criterion is not met }{
          $\bm{Z}_2 \rightarrow polar(\bm{C}_{12}^{\top}\bm{Z}_1\bm{N}) \circ \bm{T}_2$\\
          $\bm{Z}_1 \rightarrow polar(\bm{C}_{12}\bm{Z}_2\bm{N}) \circ \bm{T}_1$\\
    }
 
 \caption{ \texttt{BLOCCS} algorithm for solving Program \ref{eq:2ndStage} }
 \label{alg:l12ndstage}
\end{algorithm}

Our simulations show that for small enough $\bm{\gamma}_i, i = 1,2$ such local maximizers exist.

The same algorithm is used in the multi-modal case by maximizing over a single $\bm{Z}_i$ while keeping others constant and looping over all canonical directions. In the directed case, we use the same $\bm{\epsilon}_i$ we used in the first stage and it's again very similar to the multi-modal case. 
Although simple, we've included the corresponding algorithms for the two cases as well as algorithm for the $l_0$-regularized CCA in the \textit{Supplementals}.

\section{Experiments}
\label{sec:res}
In this section we first demonstrate performance characteristics of \texttt{BLOCCS} on simulated data; then we apply our approach to \textit{Lung Squamous Cell Carcinoma(LUSC)} multi-omics  from \textit{The Cancer Genome Atlas}\cite{weinstein2013cancer}.

\subsection{Simulated Data}
\label{subsec:sim}

Here we compare \texttt{bloccs} to \texttt{PMA} \cite{witten:tibshirani:2009}, which is a commonly used package and is a good representative of the approaches based on alternating optimization scheme which is the dominant school of approaches to the sCCA problem.
We applied both methods to the pairs of views $\bm{X}_i, i = 1,2$ estimate the first two pairs of canonical directions $\bm{Z}_i, i = 1,2$, where $\bm{X}_i \sim \mathcal{N}(\bm{0}_{p_i}, \bm{C}_{ii}), i = 1,2$, and $\bm{C}_{ii} = \bm{V}_i\bm{D}\bm{V}_i^{\top}$. We chose $p_1 = p_2, p_i/n = 10$, and constructed $\bm{V}_1 \in \mathbb{R}^{p_1 \times p_1}$ by seting up the first two columns as 
\[
\bm{v}_{11} = [\underbrace{1, \ldots, 1}_{\text{$p_1/10$}}, 0, \ldots, 0 ] ,\bm{v}_{12} = [\underbrace{0, \ldots, 0}_{\text{$p_1/10$}},\underbrace{1, \ldots, 1}_{\text{$p_1/10$}}, 0, \ldots, 0 ],
\]


and the rest of the columns by sampling according to 
\[ [\bm{V}_{1j}]_{j = 2}^{p_1} \sim \mathcal{N}(\bm{0}_{p_1-2}, \bm{I}_{p_1-2} ). \]
Similarly, $\bm{V}_2 \in \mathbb{R}^{p_2 \times p_2}$, 

\[
\bm{v}_{21} = [ 0, \ldots, 0, \underbrace{1, \ldots, 1}_{\text{$p_2/10$}}], \bm{v}_{22} = [0, \ldots, 0, \underbrace{1, \ldots, 1}_{\text{$p_2/10$}}\underbrace{0, \ldots, 0}_{\text{$p_2/10$}} ]
\]


\[
[\bm{V}_{2j}]_{j = 2}^{p_2} \sim \mathcal{N}(\bm{0}_{p_2-2}, \bm{I}_{p_2-2} )
\]

We also set $\bm{D} = diag(\sigma_1, \sigma_2, \underbrace{\sigma, \ldots, \sigma}_{\text{$p_1-2$}})$, where $\sigma_1/\sigma_2 = 2$, and $\sigma_3 = \ldots = \sigma_{p_i} = \sigma$.
We sampled $\bm{X}_i$ for 100 different values of $\sigma$, repeated 10 times, each time computing the average correlation of estimated canonical direction, $\bm{z}_{ij}$ and the underlying model, $\bm{z}_{ij} = \bm{v}_{ij}$ for $j = 1,2$, see Figure \ref{fig:cor}.a and \ref{fig:cor}.b, and also the average correlation of the first and second estimated directions, see Figure \ref{fig:cor}.c, vs. the $\lambda_3/\lambda_2$, where $\lambda_i$ is the i-th eigenvalue of the sample covariance matrix, $\bm{C}_{12}$.
It is clear from Figure \ref{fig:cor} that our approach learn the underlying model with superior accuracy while summarizing independent pieces of information in different canonical covariates. We guess that the apparent orthogonality of \texttt{PMA} estimates are mainly due to the fact that they contain minimal information about the underlying model.

\begin{figure}
    \centering
    \includegraphics[width = \textwidth]{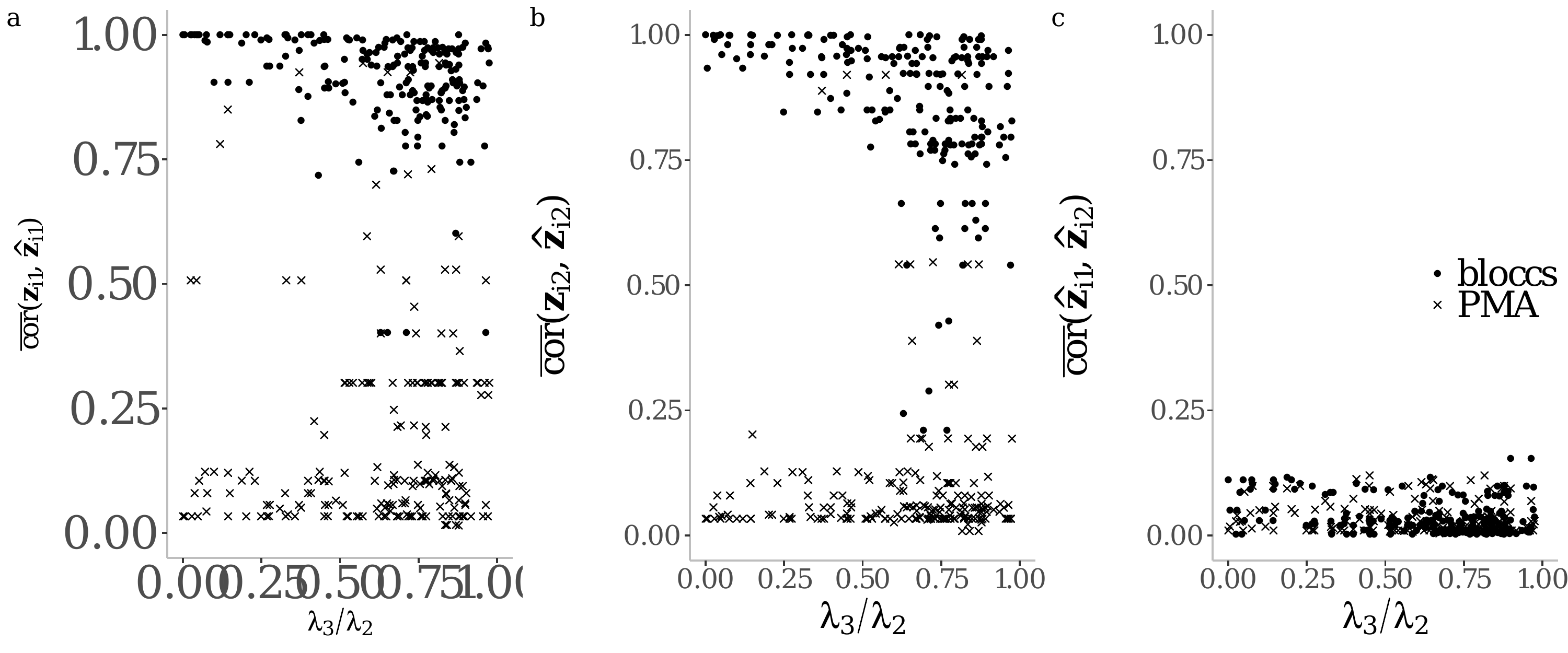}
    \caption{\textbf{a},\textbf{b}. The average correlation of the ``true", underlying model, and estimated first, and second respectively, pair of canonical directions. \textbf{c}. Average within pair correlation of the estimated directions. (plotted points are running medians). }
    \label{fig:cor}
\end{figure}

\subsection{TCGA: Lung Squamous Cell Carcinoma(LUSC)}
\label{subsec:tcga}


We first performed sCCA between methylation and RNA-expression datasets obtained via \texttt{TCGA2STAT} [\cite{wan2015tcga2stat}]. We used a permutation test, see Supplementals, for hyper-parameter tuning. While the analysis provided in \cite{wan2015tcga2stat} filters out transcripts/CpG sites with expression/methylation level falling into the 99th percentile, we didn't filter out any covariates to simulate an fully automated pipeline. Despite this disadvantage, \texttt{bloccs} also identified two distinct clusters, with (between cluster distance)/(within cluster radius) = 9.79 compared to their 2.66, as plotted in Figure \ref{fig:lusc}.a. However, contrary to their interpretation that these two groups indicate two different survival groups, as they point out the evidence against \textit{$H_0$: two survival distributions are the same} is weak; A \textit{Mantel-Cox} test returns $p-value = 0.062$, $\chi^2_1 = 3.5$ . We found out that the clusters precisely capture the \texttt{sex} effect rather than survival. We repeated the analysis, but this time we used our novel \textit{Directed sCCA} method of Algorithm \ref{alg:directed4th} with $\hat{S}(t)$ as the accessory variable. As a result we identified 25 genes and 44 CpG sites which are associated with each other and also associated with survival. Projecting the individuals onto the canonical directions, we identified two distinct clusters using \texttt{kmeans} clustering, see Figure \ref{fig:lusc}.c. We then computed the Kaplan-Meier curves for these two groups separately in Figure \ref{fig:lusc}.d. These two distributions are significantly different with $p-value = 0.0058$, $\chi^2_1 = 7.6$.

\begin{figure}
    \centering
    \includegraphics[width = .95\textwidth]{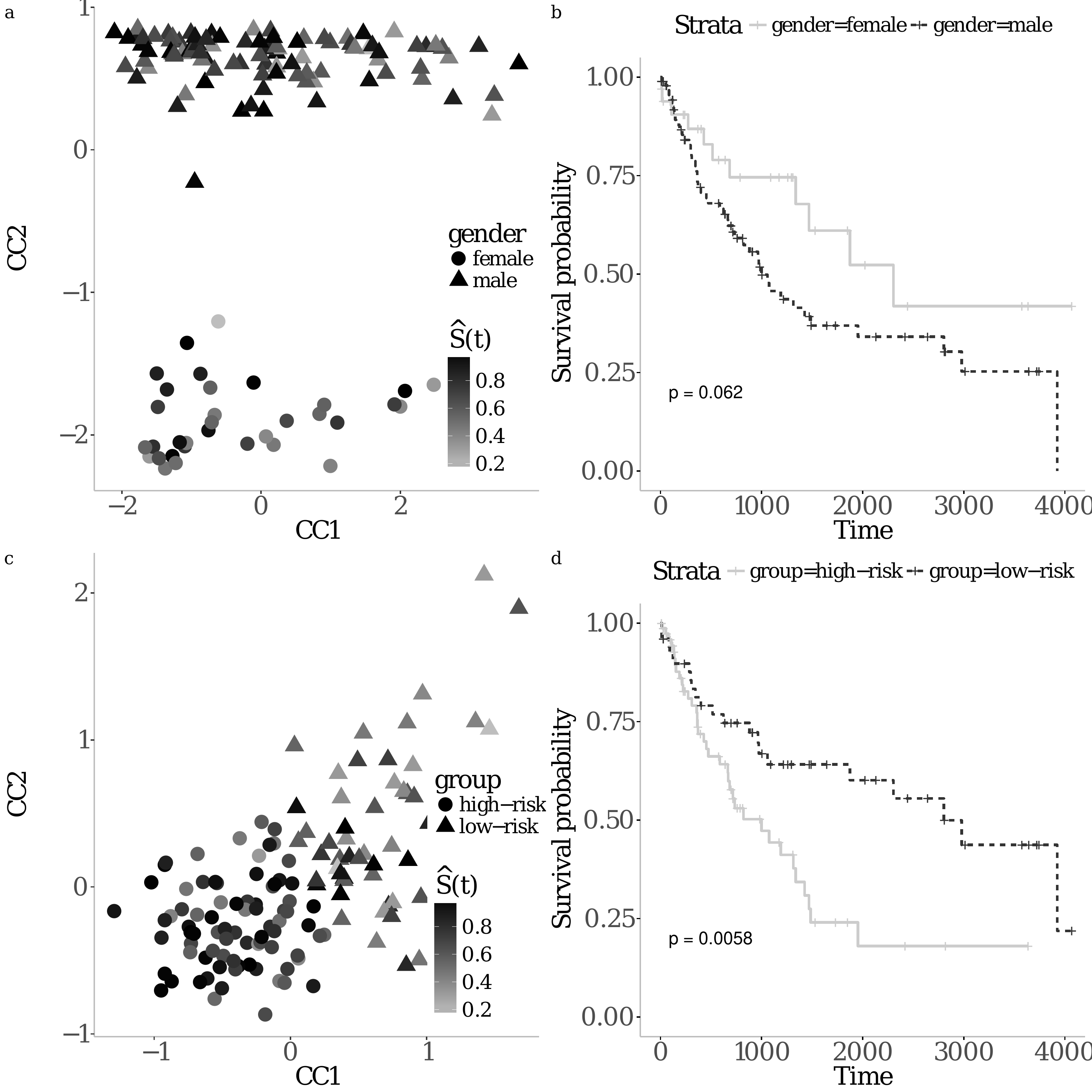}
    \caption{\textbf{a}. \texttt{kmeans} clustering of the samples projected onto the canonical directions estimated by applying \textit{sCCA} to methylation and RNA-Seq datasets for LUSC patients, shape-coded by \texttt{gender}, and color-coded by $\hat{S}(t)$, i.e. the empirical survival distribution. \textbf{b}. $\hat{S}(t)$ for the two identified groups which precisely corresponded to \texttt{gender} rather than survival propability. \textbf{c}. \texttt{kmeans} clustering of the samples projected onto the canonical directions estimated by applying \textit{Directed sCCA} to the same views and using $\hat{S}(t)$ as an accessory variable, color-coded by $\hat{S}(t)$. \textbf{d}. $\hat{S}(t)$ of the two identified groups by the Directed sCCA correspond to two significantly different, $p-value = 0.0058$, high-risk and low-risk survival groups.}
    \label{fig:lusc}
\end{figure}

\section{Conclusion}
\label{sec:conclusion}

We presented a block sparse CCA algorithm suitable for very high-dimensional settings. The method we propose and the software we provide are more stable than previous implementations of sparse CCA. Of particular interest to us is the felicity of this method to incorporate a ``guide vector" -- or an experimental design, termed \textit{accessory variables} in this article. In our lung cancer example, we included empirical survival distribution as an accessory variable, and explored genes and CpG sites that are associated with each other and patient survival probability. Indeed, we find the tuning parameters of our algorithm useful tools for data exploration, enabling the user to view a variety of relationships between views correlated more or less with an accessory variable. While multi-omics studies in biology were the motivation behind creating \texttt{bloccs}, we anticipate utility in a number of domains within and beyond the biomedical sciences.

\newpage

\appendix
\section{Proofs of Theorems}
\label{app:proofs}

\subsection{Proof of Theorem \ref{thm:l1}}

\begin{proof}

\begin{equation}
    \label{eq:l1proof1}
    \begin{split}
        \phi_{l_1, d}(\bm{\gamma}_1, \bm{\gamma}_2) &= \max_{\bm{Z}_1 \in \mathcal{S}_d^{p_1}} \sum_{j = 1}^{d} \max_{\bm{z}_{2j} \in \mathcal{S}^{p_2} } [ \mu_j \bm{z}_{1j}^{\top}\bm{C}_{12}\bm{z}_{2j} - \gamma_{2j} \| \bm{z}_{2j} \|_1 ] - \sum_{j = 1}^d \gamma_{1j} \| \bm{z}_{1j}\|_1\\
        &=\max_{\bm{Z}_1 \in \mathcal{S}_d^{p_1}} \sum_{j = 1}^{d} \max_{\bm{z}_{2j} \in \mathcal{S}^{p_2} } [ \sum_{i = 1}^{p_2} z_{2ji} (\mu_j\bm{c}_{i}^{\top}\bm{z}_{1j}) - \gamma_{2j} \| \bm{z}_{2j} \|_1 ] - \sum_{j = 1}^d \gamma_{1j} \| \bm{z}_{1j}\|_1\\
        &= \max_{\bm{Z}_1 \in \mathcal{S}_d^{p_1}} \sum_{j = 1}^{d} \max_{\bm{z}_{2j} \in \mathcal{S}^{p_2} } [ \sum_{i = 1}^{p_2} |z_{2ji}'| (\mu_j|\bm{c}_{i}^{\top}\bm{z}_{1j}| - \gamma_{2j} ) ] - \sum_{j = 1}^d \gamma_{1j} \| \bm{z}_{1j}\|_1\\
    \end{split}
\end{equation}

where $z_{2ji} = sgn(\bm{c}_i^{\top}\bm{z}_{1j})\bm{z}_{2ji}'$. Maximizing over $z_{2ji}'$ while keeping $\bm{z}_{1j}$ constant and transforming back to $z_{2ji}$, we obtain Equation \ref{eq:l1z2}. Substituting the result back in \ref{eq:l1proof1} we obtain the following optimization program,

\begin{equation}
    \label{eq:l1z1proof}
    \phi_{l_1, d}(\bm{\gamma}_1, \bm{\gamma}_2) = \max_{\bm{Z}_1 \in \mathcal{S}^{p_1}_{d}} \sum_{j = 1}^d \{ \sum_{i = 1}^{p_2} [ \mu_j |\bm{c}_i^{\top}\bm{z}_{1j}| - \gamma_{2j} ]_+^{2} - \gamma_{1j} \| \bm{z}_{1j}\|_1\}
\end{equation}

\end{proof}

\subsection{Proof of Corollary \ref{cor:l1}}

\begin{proof}
In light of Theorem \ref{thm:l1},

\begin{equation}
    \label{eq:l1z2pattern}
    z_{2ji} = 0 \Leftrightarrow [\mu_j |\bm{c}_i^{\top}\bm{z}_{1j}^*| - \gamma_{2j}]_+ = 0 \Leftrightarrow |\bm{c}_i^{\top}\bm{z}_{1j}^*| \leq \gamma_{2j}/\mu_j
\end{equation}

We can derive a sufficient condition even without solving for $\bm{Z}_1^*$ if we realize that $|\bm{c}_i^{\top}\bm{z}_{1j}^*| \leq \| \bm{c}_i \|_2 \|\bm{z}_{1j}^*\|_2 = \| \bm{c}_i \|_2$. So, $\| \bm{c}_i \|_2 \leq \gamma_{2j}/\mu_j$ is sufficient for $[\bm{T}_{2}]_{ij} = 0$.

\end{proof}

\subsection{Theorem \ref{thm:l0}}

\begin{equation}
    \label{eq:l0blockcca}
    \phi_{l_0, d}(\bm{\gamma}_1, \bm{\gamma}_2) := \max_{\substack{\bm{Z}_1 \in \mathcal{S}_d^{p_1}\\ \bm{Z}_2 \in \mathcal{S}_d^{p_2} } } tr(  diag(\bm{Z}_1^{\top}\bm{C}_{12}\bm{Z}_2\bm{N})^2 ) - \sum_{j = 1}^d \gamma_{1j} \| \bm{z}_{1j} \|_0 - \sum_{j = 1}^d \gamma_{2j} \| \bm{z}_{2j} \|_0
\end{equation}

where as before $\bm{N} = diag(\mu_1, \ldots, \mu_d) \succ 0$, and $\gamma_{ij} \geq 0$.

\begin{theorem}
\label{thm:l0}
The solutions $\bm{Z}_1^*$ and $\bm{Z}_2^*$ of the optimization program \ref{eq:l0blockcca} is given by,

\begin{equation}
    \label{eq:l0z1}
    \bm{Z}_1^* = \argmax_{\bm{Z}_1 \in \mathcal{S}^{p_1}_{d}} \sum_{j = 1}^d \sum_{i = 1}^{p_2} [(\mu_j \bm{c}_i^{\top} \bm{z}_{1j})^2 - \gamma_{2j}]_+ - \sum_{j = 1}^d \gamma_{1j} \| \bm{z}_{1j} \|_0
\end{equation}

and,

\begin{equation}
    \label{eq:l0z2}
    [\bm{Z}_2]_{ij}^* = \frac{[sgn((\mu_j \bm{c}_i^{\top} \bm{z}_{1j})^2 - \gamma_{2j})]_+ \mu_j \bm{c}_i^{\top} \bm{z}_{1j}}{\sqrt{\sum_{k = 1}^{p_2} [sgn((\mu_j \bm{c}_k^{\top} \bm{z}_{1j})^2 - \gamma_{2j})]_+ (\mu_j \bm{c}_k^{\top} \bm{z}_{1j})^2}}
\end{equation}

\end{theorem}

\begin{proof}
Maximization problem \ref{eq:l0blockcca} can be decoupled along different canonical directions as the following optimization problem over $\bm{Z}_1$,

\begin{equation}
    \label{eq:thml0}
    \phi_{l_0, d}(\bm{\gamma}_1, \bm{\gamma}_2) =  \max_{\bm{Z}_1 \in \mathcal{S}_d^{p_1}} \sum_{j = 1}^d \max_{\bm{z}_{2j} \in \mathcal{S}_d^{p_2}}[ (\mu_j \bm{z}_{1j}^{\top} \bm{C}_{12} \bm{z}_{2j})^2  - \gamma_{2j} \| \bm{z}_{2j} \|_0 ]  - \sum_{j = 1}^d \gamma_{1j} \| \bm{z}_{1j} \|_0 
\end{equation}

As in Theorem \ref{thm:l1}, we first solve for $\bm{z}_{2j}$ while keeping $\bm{Z}_1$ constant, resulting in Equation \ref{eq:l0z2}. The reason is that $z_{2ji} \neq 0$ only if the maximum objective value $(\mu_j \bm{c}_{i}^{\top} \bm{z}_{1j})^2  - \gamma_{2j}$ is positive. Now replacing back in \ref{eq:thml0} we obtain,

\begin{equation}
    \label{eq:l0proof}
    \phi_{l_0, d}(\bm{\gamma}_1, \bm{\gamma}_2) = \max_{\bm{Z}_1 \in \mathcal{S}^{p_1}_{d}} \sum_{j = 1}^d \{\sum_{i = 1}^{p_2} [(\mu_j \bm{c}_i^{\top} \bm{z}_{1j})^2 - \gamma_{2j}]_+ - \gamma_{1j} \| \bm{z}_{1j} \|_0\}
\end{equation}

\end{proof}

\begin{corollary}
\label{cor:l0}
$[\bm{T}_{2}]_{ij} = 0$, i.e. $z_{2ji}^* \in supp(\bm{Z}_2^*)$, iff $(\bm{c}_i^{\top} \bm{z}_{1j}^*)^2 \leq \gamma_{2j}/\mu_j^2$.
\end{corollary}

\begin{proof}
According to Theorem \ref{thm:l0},

\begin{equation}
    \label{eq:l0z2pattern}
    z_{2ji}^* = 0 \Leftrightarrow [(\mu_j \bm{c}_i^{\top} \bm{z}_{1j}^*)^2 - \gamma_{2j}]_+ = 0 \Leftrightarrow ( \bm{c}_i^{\top} \bm{z}_{1j}^*)^2  \leq \gamma_{2j}/\mu_j^2
\end{equation}

We can again derive a sufficient condition by just realizing that $( \bm{c}_i^{\top} \bm{z}_{1j}^*)^2 \leq \| \bm{c}_i \|_2^2 \|\bm{z}_{1j}^*\|_2^2 = \| \bm{c}_i \|_2^2$. So, $\| \bm{c}_i \|_2^2 \leq \gamma_{2j}/\mu_j^2$ is sufficient for $[\bm{T}_{2}]_{ij} = 0$.

\end{proof}

\begin{remark}
\label{app:rmk:approx}
According to Theorem \ref{thm:l0}, in order to infer the sparsity pattern matrices, we need to optimize Program \ref{eq:l0proof} depending on the regularization of choice. This program is non-convex; however we approximate it by ignoring the penalty term which turns it into the following concave minimization programs over the unit sphere, 

\begin{equation}
    \label{eq:l0approx}
    \phi_{l_0, d}(\bm{\gamma}_1, \bm{\gamma}_2) = \max_{\bm{Z}_1 \in \mathcal{S}^{p_1}_{d}} \sum_{j = 1}^d \{\sum_{i = 1}^{p_2} [(\mu_j \bm{c}_i^{\top} \bm{z}_{1j})^2 - \gamma_{2j}]_+ \}
\end{equation}

which is solved using a simple gradient ascent algorithm. It is important to note that this approximation is very reasonable and justifiable. Our simulations demonstrate that this approximation does not affect the capability of our approach to precisely uncover the support of our underlying generative model. Secondly, as we have mentioned in Corollary \ref{cor:l0}, we use the optima of this program in the first stage to infer the sparsity pattern of the canonical direction on the other side. Also we can show that for every $(\gamma_{1j}, \gamma_{2j})$ that results in $\bm{z}_{1j}^*$ for which \ref{eq:l0z2pattern} holds and as a result $z_{2ji}^* = 0$ in Program \ref{eq:l0proof}, one can find a $\gamma_{2j}' \geq \gamma_{2j}$ in Program \ref{eq:l0approx} for which $z_{2ji}^* = 0$.
\end{remark}

\subsection{Proof of Theorem \ref{thm:mCCA}}
\label{app:subsec:multimodal}

\begin{proof}
\begin{alignat}{2}
\label{eq:mccaformulation}
\phi_{l_1,d}^m(\bm{\Gamma}_1, \ldots, \bm{\Gamma}_d) &= \max_{\substack{\bm{Z}_r \in \mathcal{S}_d^{p_r}\\ r \neq s, r = 1, \ldots, m }} \max_{\bm{Z}_s \in \mathcal{S}_d^{p_s}} && \sum_{r<s = 2}^{m} tr(\bm{Z}_r^{\top}\bm{C}_{rs}\bm{Z}_s\bm{N}) - \sum_{j = 1}^d \sum_{s = 2}^m  \sum_{\substack{r = 1 \\ r \neq s }}^{s-1} \gamma_{srj} \| \bm{z}_{sj} \|_1 \\
&= \max_{\substack{\bm{Z}_r \in \mathcal{S}_d^{p_r}\\ r \neq s, r = 1, \ldots, m }} \sum_{j = 1}^d [\max_{\bm{z}_{sj} \in \mathcal{S}^{p_s}} && \sum_{ r < s = 2}^{m-1} \mu_j \bm{z}_{rj}^{\top}\bm{C}_{rs}\bm{z}_{sj} - \sum_{s = 1}^m  \sum_{\substack{r = 1 \\ r \neq s }}^{m-1} \gamma_{srj} \| \bm{z}_{sj} \|_1]\\
\nonumber
&= \max_{\substack{\bm{Z}_r \in \mathcal{S}_d^{p_r}\\ r \neq s, r = 1, \ldots, m }} \sum_{j = 1}^d [\max_{\bm{z}_{sj} \in \mathcal{S}^{p_s}} && \sum_{i = 1}^{p_s} z_{sij}(\sum_{ \substack{r = 1\\ r \neq s}}^m  \tilde{\bm{c}}_{rsi}^{\top} \bm{z}_{rj}) - \sum_{\substack{r = 1 \\ r \neq s }}^{m} \gamma_{srj} \| \bm{z}_{sj} \|_1] +\\
& && \overbrace{\sum_{\substack{i < j = 2\\ i, j \neq s }}^{m} tr(\bm{Z}_r^{\top}\bm{C}_{rs}\bm{Z}_s\bm{N}) - \sum_{j = 1}^d \sum_{\substack{i = 1\\ i \neq s }}^m  \sum_{\substack{r = 1 \\ i \neq r }}^{i-1} \gamma_{irj} \| \bm{z}_{ij} \|_1}^\text{\textit{I}}\\
\label{eq:mccaformulationlastline}
&= \max_{\substack{\bm{Z}_r \in \mathcal{S}_d^{p_r}\\ r \neq s, r = 1, \ldots, m }} \sum_{j = 1}^d [\max_{\bm{z}_{sj} \in \mathcal{S}^{p_s}} && \sum_{i = 1}^{p_s} |z_{sij}'| (|\sum_{ \substack{r = 1\\ r \neq s}}^m  \tilde{\bm{c}}_{rsi}^{\top} \bm{z}_{rj}| - \sum_{\substack{r = 1 \\ r \neq s }}^{m} \gamma_{srj} )] + I
\end{alignat}

where the last line follows from $z_{sij} = sgn(\sum_{ \substack{r = 1\\ r \neq s}}^m  \tilde{\bm{c}}_{rsi}^{\top} \bm{z}_r)z_{sij}'$. $\tilde{\bm{c}}_{rsi} = \bm{c}_{rsi}$ if $r< s$, and $\tilde{\bm{c}}_{rsi} = \bm{c}_{rsi}^{\top}$ if $r > s$ where $\bm{c}_{rsi}$ is the $i$th row of $\bm{C}_{rs} = 1/n \bm{X}_r^T \bm{X}_s$. Now solving for $\bm{z}_{sj}'$ and translating back to $\bm{z}_{sj}$ and normalizing, we get the solution in \ref{eq:mZsstar}. Substituting this solution back to \ref{eq:mccaformulationlastline},

\begin{equation}
\label{eq:app:mZR}
\begin{split}
    \phi_{l_1,d}^{2m}(\bm{\Gamma}_1, \ldots, \bm{\Gamma}_d) = \max_{\substack{\bm{Z}_r \in \mathcal{S}^{p_r}_d\\ r \neq s, r = 1, \ldots, m} } & \sum_{j = 1}^d \sum_{i= 1}^{p_s} [\mu_j|\sum_{ \substack{r = 1\\ r \neq s}}^m  \tilde{\bm{c}}_{rsi}^{\top} \bm{z}_{rj}| - \sum_{\substack{r = 1\\ r \neq s}}^{m} \gamma_{srj}]_+^2 +\\  
    &\sum_{ \substack{i < r = 2\\ i, r \neq s }}^m tr(\bm{Z}_i^{\top}\bm{C}_{ir}\bm{Z}_r\bm{N})  -\sum_{j = 1}^d\sum_{\substack{i = 1 \\ i \neq s} }^m  \sum_{\substack{r = 1 \\ i \neq j }}^{s-1} \gamma_{irj} \| \bm{z}_{ij} \|_1
\end{split}
\end{equation}

\end{proof}

\subsection{Proof of Corollary \ref{cor:mccasparsityl0}}

\begin{proof}
Utilizing the results in Equation \ref{eq:mZsstar},

\begin{equation}
\label{eq:multisparsityl0}
  [\bm{Z}_s]_{ij}^* = 0 \Leftrightarrow  [\mu_j|\sum_{ \substack{r = 1\\ r \neq s}}^m  \tilde{\bm{c}}_{rsi}^{\top} \bm{z}_{rj}| - \sum_{\substack{r = 1\\ r \neq s}}^{m} \gamma_{srj}]_+ = 0 \Leftrightarrow \mu_j|\sum_{ \substack{r = 1\\ r \neq s}}^m  \tilde{\bm{c}}_{rsi}^{\top} \bm{z}_{rj}| \leq \sum_{\substack{r = 1\\ r \neq s}}^{m} \gamma_{srj}
\end{equation}

and as before we can identify a more general sufficient condition regardless of $\bm{Z}_r^*$,

\begin{equation}
   \mu_j|\sum_{ \substack{r = 1\\ r \neq s}}^m  \tilde{\bm{c}}_{rsi}^{\top} \bm{z}_{rj}| \leq \mu_j \sum_{ \substack{r = 1\\ r \neq s}}^m \| \tilde{\bm{c}}_{rsi} \|_2 \| \bm{z}_{rj} \|_2 = \sum_{ \substack{r = 1\\ r \neq s}}^m \| \tilde{\bm{c}}_{rsi} \|_2
\end{equation}

Hence, $[\bm{T}_{s}]_{ij} = 0$ if $\sum_{ \substack{r = 1\\ r \neq s}}^m \| \tilde{\bm{c}}_{rsi} \|_2 \leq \sum_{\substack{r = 1\\ r \neq s}}^{m} \gamma_{srj} $ regardless of $\bm{Z}_r^*$. 
\end{proof}

\subsection{Proof of Theorem \ref{thm:directedl1}}
\label{app:subsec:directed}

\begin{proof}
\begin{equation}
    \label{eq:directedl1proof1}
    \begin{split}
        \phi_{l_1, d}(\bm{\gamma}_1, \bm{\gamma}_2) &= \max_{\bm{Z}_1 \in \mathcal{S}_d^{p_1}} \sum_{j = 1}^{d} \max_{\bm{z}_{2j} \in \mathcal{S}^{p_2} } [ \sum_{i = 1}^{p_2} z_{2ji} \mu_j(\bm{c}_{i}^{\top}\bm{z}_{1j}+ \epsilon_{2j}\bm{x}_{2i}^{\top}\bm{y}_j) - \gamma_{2j} \| \bm{z}_{2j} \|_1 ] \\ 
        &+\sum_{j = 1}^d (\mu_j\epsilon_{1j}\bm{y}_j^{\top}\bm{X}_1\bm{z}_{1j} - \gamma_{1j} \| \bm{z}_{1j}\|_1)\\
        &= \max_{\bm{Z}_1 \in \mathcal{S}_d^{p_1}} \sum_{j = 1}^{d} \max_{\bm{z}_{2j} \in \mathcal{S}^{p_2} } [ \sum_{i = 1}^{p_2} |z_{2ji}'| (\mu_j|\bm{c}_{i}^{\top}\bm{z}_{1j}+\epsilon_{2j}\bm{x}_{2i}^{\top}\bm{y}_j| - \gamma_{2j} ) ]\\
        &+\sum_{j = 1}^d (\mu_j\epsilon_{1j}\bm{y}_j^{\top}\bm{X}_1\bm{z}_{1j} - \gamma_{1j} \| \bm{z}_{1j}\|_1)\\
    \end{split}
\end{equation}

Similar to Theorem \ref{thm:l1}, $z_{2ji} = sgn(\bm{c}_{i}^{\top}\bm{z}_{1j} +\epsilon_{2j}\bm{x}_{2i}^{\top}\bm{y}_j)z_{2ji}'$. Maximizing over $z_{2ji}'$ while keeping $\bm{z}_{1j}$ constant and transforming back to $z_{2ji}$, we obtain Equation \ref{eq:directedl1z2}. Substituting the result back in Program \ref{eq:l1proof1} we obtain the following optimization program,

\begin{equation}
    \label{eq:directedl1z1proof}
    \phi_{l_1, d}(\bm{\gamma}_1, \bm{\gamma}_2) = \max_{\bm{Z}_1 \in \mathcal{S}^{p_1}_{d}} \sum_{j = 1}^d \{ \sum_{i = 1}^{p_2} [ \mu_j|\bm{c}_{i}^{\top}\bm{z}_{1j}+\epsilon_{2j}\bm{x}_{2i}^{\top}\bm{y}_j| - \gamma_{2j} ]_+^{2}+ \mu_j\epsilon_{1j}\bm{y}_j^{\top}\bm{X}_1\bm{z}_{1j} - \gamma_{1j} \| \bm{z}_{1j}\|_1\}
\end{equation}
\end{proof}

\subsection{Proof of Corollary \ref{cor:directedl1}}

\begin{proof}
Per Equation \ref{eq:directedl1z2},

\begin{equation}
    \label{eq:directedl1z2pattern}
    z_{2ij} = 0 \Leftrightarrow [\mu_j |\bm{c}_k^{\top}\bm{z}_{1j}+\epsilon_{2j}\bm{x}_{2k}^{\top}\bm{y}_j| - \gamma_{2j}]_+ = 0 \Leftrightarrow |\bm{c}_k^{\top}\bm{z}_{1j}^*+\epsilon_{2j}\bm{x}_{2k}^{\top}\bm{y}_j| \leq \gamma_{2j}/\mu_j
\end{equation}

More generally in order for $[\bm{T}_{2}]_{ij} = 0$, it is sufficient to have $\| \bm{c}_i^{\top} \|_2 \leq \gamma_{2j}/\mu_j$ since $|\bm{c}_k^{\top}\bm{z}_{1j}^*+\epsilon_{2j}\bm{x}_{2k}^{\top}\bm{y}_j| \leq \| \bm{c}_i \|_2 \|\bm{z}_{1j}^*\|_2 + \epsilon_{2j}\|\bm{x}_{2k}\|_2 \|\bm{y}_j\|_2 = \| \bm{c}_i \|_2 +\epsilon_{2j}\|\bm{x}_{2k}\|_2$ assuming $\bm{y}_j$ is normalized.
\end{proof}

\section{Algorithms}
\label{app:algs}

\subsection{First Order Optimization Method}
\label{app:alg:gd}

\begin{algorithm}[H]
\KwData{$z_0 \in \mathcal{Q}$}
\KwResult{$z_k^*  = \argmax_{z\in \mathcal{Q}} f(z) $}
$k \leftarrow 0$\\
\While{convergence criterion is not met}{
	$z_{k+1} \leftarrow \argmax_{x \in \mathcal{Q}} (f(z_k) + (x - z_k)^Tf'(z_k)) $ \\
	$k \leftarrow k +1$
}
\caption{A first-order optimization method.}
\label{alg:1stOrder}
\end{algorithm}

\subsection{\texorpdfstring{$L_0$}{TEXT} Regularized Algorithm}
\label{subsec:l0alg}
Now we apply our first-order maximization algorithm to Program \ref{eq:l0approx},

\begin{algorithm}[H] 
 \KwData{Sample Covariance Matrix $\bm{C}_{12}$\\  \quad \qquad    Regularization parameter vector $\bm{\gamma_2} \in [0,1]^{d}$ \\ \quad \qquad Initialization $\bm{Z}_1 \in \mathcal{S}_d^{p_1}$\\ \quad \qquad $\bm{N} = diag(\mu_1, \ldots, \mu_d) \succ 0$\\
 \quad \qquad (optional) $\bm{T}_1 \in \{0, 1\}^{p_1 \times d}$}
 \KwResult{ $\bm{T}_2$, optimal sparsity pattern of $\bm{Z}_2^*$}
 initialization\;
 
 \While{ convergence criterion is not met }{
          \For{$j = 1, \ldots, d$}{$\bm{z}_{1j} \leftarrow \sum_{i= 1}^{p_2} \mu_j^2[ ( \mu_j \bm{c}_i^{\top} \bm{z}_1 )^2 - \gamma_2  ]_+ \bm{c}_i^{\top}\bm{z}_1\bm{c}_i$\\}
    $\bm{Z}_1 \leftarrow polar(\bm{Z}_1)$\\
    \If{$\bm{T}_1$ is given}{$\bm{Z}_1 \leftarrow \bm{Z}_1 \circ \bm{T}_1$}
    }
 
 Output $\bm{T}_2 \in \{0,1\}^{p_2 \times d}$ where $[\bm{T}_2]_{ij} = 0 $ if $(\bm{c}_i^{\top}\bm{z}_{1j}^*)^2 \leq \gamma_{2j}/\mu_j^2$ and 1 otherwise.\\
 
 \caption{ \texttt{BLOCCS} algorithm for solving Program \ref{eq:l0approx} }
 \label{alg:l03rd}
\end{algorithm}

\vspace{\baselineskip}

There won't be a second stage here, since finding $\bm{T}_i$ is the final goal.

\subsection{Active Entry Estimation For Multi-Modal sCCA}
\label{app:alg:multimodal}

\begin{equation}
    \label{eq:app:2ndStagemCCA}
    \phi_{d, 0}^m = \max_{\substack{\bm{Z}_r \in \mathcal{S}_d^{p_r}, r = 1,\ldots, m\\ \bm{Z}_r|_{\neq 0} = \bm{T}_r}} \sum_{r < s = 2}^{m} tr( \bm{Z}_r^{\top}\bm{C}_{rs}\bm{Z}_s\bm{N} )
\end{equation}

\vspace{\baselineskip}
\begin{algorithm}[H] 
 \KwData{Sample Covariance Matrices $\bm{C}_{rs}, \quad 1 \leq r < s \leq m$\\ \quad \qquad Initial values $\bm{Z}_r \in \mathcal{S}_d^{p_r}, \quad 1 \leq r \leq m$\\ \quad \qquad $\bm{N} = diag(\mu_1, \ldots, \mu_d) \succ 0$\\
 \quad \qquad $\bm{T}_r \in \{0, 1\}^{p_r \times d}, r \neq s$ \quad \qquad }
 \KwResult{ $\bm{Z}_i^*, i = 1,\ldots, m$, local maximizers of \ref{eq:app:2ndStagemCCA}}
 initialization\;
 \While{ convergence criterion is not met }{
 \For{$s = 1, \ldots, m$}{
    $\bm{Z}_s \rightarrow polar(\sum_{\substack{r = 1\\ r\neq s} }^m \tilde{\bm{C}}_{rs}^{\top}\bm{Z}_r\bm{N})$\\
          $\bm{Z}_{s} \rightarrow \bm{Z}_{s} \circ \bm{T}_s$
    }
 }
 
 \caption{\texttt{BLOCCS} algorithm for solving Program \ref{eq:app:2ndStagemCCA}}
 \label{alg:mCCA3rd2ndStage}
\end{algorithm}
\vspace{\baselineskip}

\subsection{Active Entry Estimation For Directed sCCA}
\label{app:alg:directed}

We estimate active entries of the canonical directions in the second stage via the following maximization program,

\begin{equation}
    \label{eq:app:2ndStageDirected}
    \phi_{d, 0} = \max_{\substack{\bm{Z}_1 \in \mathcal{S}_d^{p1}, \bm{Z}_1|_{\neq 0} = \bm{T}_1\\
    \bm{Z}_2 \in \mathcal{S}_d^{p2}, \bm{Z}_2|_{\neq 0} = \bm{T}_2}} tr( \bm{Z}_1^{\top}\bm{C}_{12}\bm{Z}_2\bm{N} ) + \sum_{i = 1}^2 tr(\bm{Y}^{\top}\bm{X}_i\bm{Z}_{i}\bm{N}\bm{E}_i)
\end{equation}

\vspace{\baselineskip}

\begin{algorithm}[H] 
 \KwData{Sample Covariance Matrix $\bm{C}_{12}$\\ \quad \qquad Initialization $\bm{Z}_i \in \mathcal{S}_d^{p_i}$ ,$i = 1,2$\\ \quad \qquad $\bm{N} = diag(\mu_1, \ldots, \mu_d) \succ 0$\\
 \quad \qquad $\bm{T}_i \in \{0, 1\}^{p_i \times d}$, $i  = 1,2$\\ \quad \qquad $\bm{E}_i = diag(\bm{\epsilon_i}), i = 1,2$}
 \KwResult{ $\bm{Z}_i^*, i = 1,2$, local maximizers of \ref{eq:app:2ndStageDirected}}
 initialization\;
 
 \While{ convergence criterion is not met }{
          $\bm{Z}_2 \rightarrow polar(\bm{C}_{12}^{\top}\bm{Z}_1\bm{N} + \bm{X}_2^{\top}\bm{Y}\bm{N}\bm{E}_2) \circ \bm{T}_2$\\
          $\bm{Z}_1 \rightarrow polar(\bm{C}_{12}\bm{Z}_2\bm{N} + \bm{X}_1^{\top}\bm{Y}\bm{N}\bm{E}_1) \circ \bm{T}_1$\\
    }
 
 \caption{ \texttt{BLOCCS} algorithm for solving Program \ref{eq:app:2ndStageDirected} }
 \label{alg:l12ndstageDirected}
\end{algorithm}

\vspace{\baselineskip}

\subsection{Hyper-parameter Tuning Using Permutation Test}
\label{app:hyperparameter}

\begin{algorithm}[H]
 \KwData{Sample matrices $\bm{X}_i \in \mathbb{R}^{n \times p_i}$, $i = 1,2$ \\  \quad \qquad     Sparsity parameters $\gamma_i$, $i = 1,2$\\ \quad \qquad Initial values $\bm{z}_i \in \mathcal{S}^{p_i}$, $i = 1,2$ \\  
 \quad \qquad Number of permutations $P$}
 \KwResult{ $p_{\gamma_1, \gamma_2}$ the evidence against the null hypothesis that the canonical correlation is not lower when $X_i$ are independent.}

Compute $({\bm{z}_1^*}, {\bm{z}_2^*})$ on $\bm{X}_{1}, \bm{X}_{2}$ via any of the proposed algorithms with sparsity hyperparameters $(\gamma_{1}, \gamma_{2})$\\
$\rho(\gamma_1, \gamma_2) = corr(\bm{X}_{1}{\bm{z}_1^*}, \bm{X}_{2}{\bm{z}_2^*})$\\

 \For{p = 1, \ldots, P}{
 Let $\bm{X}_{1}^{(p)}$ be a row-wise permutation of $\bm{X}_1$\\
 Compute $({\bm{z}_1^*}^{(p)}, {\bm{z}_2^*}^{(p)})$ on $(\bm{X}_{1}^{(p)}, \bm{X}_{2})$ via any of the proposed algorithms with sparsity hyperparameters $(\gamma_{1}, \gamma_{2})$\\
 $\rho_{perm}^{(p)}(\gamma_1, \gamma_2) = corr(\bm{X}_{1}^{(p)}{\bm{z}_1^*}^{(p)}, \bm{X}_{2}{\bm{z}_2^*}^{(p)})$
 }
 $p_{\gamma_1, \gamma_2} = 1/P \sum_{p = 1}^P I(\rho_{perm}^{(p)} > \rho)$ \\
 
\caption{Hyperparameter Tuning via Permutation Test}
\label{alg:hyperparameterpermutation}
\end{algorithm}

\vskip 0.2in
\bibliography{jmlr}

\end{document}